\documentclass{article} 
\usepackage[final]{colm2026_conference}

\usepackage{microtype}
\usepackage{hyperref}
\usepackage{url}
\usepackage{booktabs}
\usepackage{graphicx}
\usepackage{multirow}
\usepackage{arydshln}
\usepackage{amsmath}
\usepackage{amsthm}
\newtheorem{theorem}{Theorem}
\newtheorem{definition}{Definition}[section]
\usepackage{algorithm}
\usepackage{algorithmic}
\usepackage{lineno}
\usepackage[table]{xcolor}

\newcommand{\rev}[1]{\textcolor{revblue}{#1}}
\newcommand{\revon}{\color{revblue}}
\definecolor{revblue}{rgb}{0.00, 0.0, 0.0}

\definecolor{darkblue}{rgb}{0, 0, 0.5}
\hypersetup{colorlinks=true, citecolor=darkblue, linkcolor=darkblue, urlcolor=darkblue}

\title{RankGuide: Tensor-Rank-Guided Routing and Steering for Efficient Reasoning}

\author{
Jiayi Tian$^{1}$, Yupeng Su$^{1}$, Ryan Solgi$^{1}$, Souvik Kundu$^{2}$, Zheng Zhang$^{1}$ \\
$^{1}$University of California, Santa Barbara \\
$^{2}$Intel Labs \\
\\
\texttt{\{jiayi\_tian, yupengsu, solgi\}@ucsb.edu, zhengzhang@ece.ucsb.edu, souvikk.kundu@intel.com}
}

\begin{document}

\ifcolmsubmission
\linenumbers
\fi

\maketitle

\begin{abstract}
Large reasoning models (LRMs) enhance problem-solving capabilities by generating explicit multi-step chains of thought (CoT) reasoning; however, they incur substantial inference latency and computational overhead.
To mitigate this issue, recent works have explored model collaboration paradigms, where small reasoning models (SRMs) generate intermediate reasoning steps to achieve a better accuracy–latency trade-off.
Despite recent progress, effectively and efficiently detecting and mitigating SRM failures in collaborative systems remains a key challenge.
To address this issue, we analyze SRM inference in both the generated text and hidden-state spaces, and identify three types of failure modes: \textit{overconfidence}, \textit{uncertainty}, and \textit{heavy revalidation}.
Building on these insights, we propose \textbf{RankGuide}, a framework that improves the efficiency and effectiveness of SRM–LRM collaboration through tensor-rank-guided routing and steering. 
Specifically, RankGuide leverages a routing signal that incorporates \textit{tensor-rank signals} derived from consecutive hidden states to detect when SRMs are likely to fail and selectively invoke LRMs. 
In addition, we introduce a tensor-rank-filtered steering vector extraction method to modulate the reasoning trajectory of SRMs, thereby improving their generation quality.
By improving both routing and steering through tensor-rank signals, RankGuide enables SRM–LRM collaborative systems to achieve more efficient reasoning with fewer steps and improved accuracy.
Experiments \rev{across three reasoning domains -- mathematics, code generation, and scientific QA --} demonstrate the efficacy of RankGuide in reducing latency by up to $1.75\times$ compared to LRM, while maintaining competitive accuracy relative to prior methods. The code is available at \href{https://github.com/TTTTTTris/RankGuide}{https://github.com/TTTTTTris/RankGuide}.
\end{abstract}

\section{Introduction}
Large reasoning models (LRMs) \citep{guo2025deepseek, yang2025qwen3} have recently demonstrated strong performance on complex problem-solving tasks by explicitly generating multi-step chains of thought (CoT) reasoning \citep{wei2022chain}. By decomposing problems into intermediate steps, these models achieve significant gains in domains such as mathematical reasoning \citep{hendrycks2measuring, aime24} and code generation \citep{jain2025livecodebench}. However, this improved capability comes at the cost of substantially increased inference latency and computational overhead, limiting their practicality in real-world deployment scenarios.

To mitigate this issue, recent works \citep{shi-etal-2025-speccot, panspecreason, fuscaling, zeng2026glimprouter} have explored model collaboration paradigms, where small reasoning models (SRMs) generate intermediate reasoning steps while larger models are invoked selectively. 
Despite these advances, \textit{effectively and efficiently detecting and mitigating SRM failures} remains a key challenge in collaborative systems. 
For instance, SpecReason \citep{panspecreason} relies on LLM verifier feedback to decide whether to let larger models regenerate the reasoning step, which introduces additional computational and latency overhead.
Recently, GlimpRouter \citep{zeng2026glimprouter} triggers routing prior to each step based on entropy-derived uncertainty signals, thereby avoiding the need for LLM verifiers or rollback mechanisms. 
However, we observe that such approaches can fail when SRMs produce overconfident yet incorrect reasoning steps, resulting in missed routing opportunities and confident errors. 
Consequently, inaccurate failure detection allows errors to propagate through the reasoning chain, degrading final performance and diminishing efficiency gains.

To better understand and address these limitations, we conduct a systematic analysis of SRM inference from the perspective of the generated text as well as the hidden-states. 
Our analysis reveals three characteristic failure modes: (a) \textit{overconfidence}, where the model produces confident but incorrect reasoning; (b) \textit{uncertainty}, where the model struggles to make progress; and (c) \textit{heavy revalidation}, where the model repeatedly revisits prior steps without advancing the solution. 
These findings highlight the need for more reliable signals to detect diverse SRM failure modes and guide appropriate interventions during inference.

Building on these insights, we propose \textbf{RankGuide}, a \textit{training-free} serving framework that improves the efficiency and effectiveness of SRM--LRM collaboration.
Our method leverages \textit{tensor-rank signals} derived from consecutive hidden states to guide both routing and steering for efficient LLM inference.
\rev{Concretely, the signal is used to filter the calibration dataset for steering vector generation offline, while guiding the model routing during decoding.}
Our contributions are summarized as follows:

\begin{itemize}
    \item 
    We identify SRM failures as the primary bottleneck in collaborative reasoning systems, and provide a systematic analysis of SRM inference from both generated-text and hidden-state perspectives, revealing three key failure modes: \textit{overconfidence}, \textit{uncertainty}, and \textit{heavy revalidation}.
    
    \item 
    We propose tensor-rank-guided routing, which employs a routing policy driven by a composite signal integrating token-level entropy and tensor-rank features derived from hidden states to identify potential SRM failures and selectively invoke LRMs, addressing both uncertainty and overconfident errors.
    
    \item 
    We introduce a tensor-rank-filtered steering vector extraction method that modulates the hidden-state trajectory of SRMs during inference using a steering vector computed from a tensor-rank-filtered calibration set, thereby reducing heavy revalidation in SRM generation.
\end{itemize}
Experimental results \rev{across three reasoning benchmarks spanning mathematics, code generation and scientific QA} demonstrate that RankGuide can reduce latency and computational cost while maintaining competitive accuracy compared to prior methods. Specifically, RankGuide can yield up to $1.75\times$ and $1.36\times$ latency benefit compared to LRM and SoTA collaborative inference frameworks, while maintaining or improving the accuracy. 

\section{Related Works}
\subsection{Tensor Decomposition and Low-Rank Structure}
Low-rank structures have been widely studied in deep learning for model compression and efficient representation learning, including singular value decomposition (SVD) \citep{wall2003singular, kang2024gear} for matrices and tensor-train (TT) decomposition \citep{oseledets2011tensor} for high-dimensional tensors, among other techniques. Prior works primarily leverage tensor decompositions to approximate weight matrices or adapters \citep{tian2025flat, lilestd, yang2024loretta}, thereby reducing memory and computational cost.
In contrast, our work does not use tensor decomposition as a compression tool, \rev{but instead utilizes} tensor-rank as a diagnostic signal for reasoning complexity, enabling representation-driven routing and improved steering in LRM inference.

\subsection{Model Collaboration}
Model collaboration \citep{wang2023fusing, zheng2025citer} can improve inference efficiency by combining models of different capacities, where small reasoning models (SRMs) handle simple cases and large reasoning models (LRMs) are invoked for more complex reasoning. Recent works extend this paradigm to LLM reasoning by allowing SRMs to generate intermediate steps with selective verification or refinement.
A central challenge is effective routing between SRMs and LRMs. Existing methods often rely on LLM verifier \citep{panspecreason, shi-etal-2025-speccot, fuscaling}, which is costly and may lead to suboptimal decisions. To address this, GlimpRouter \citep{zeng2026glimprouter} proposes an entropy-based confidence score that routes uncertain steps based on the SRM's output distribution, achieving improved efficiency.
However, relying solely on the SRM's output confidence as the routing signal can be biased, as SRMs may be both highly confident and incorrect. To address this limitation, we propose RankGuide, which leverages tensor-rank signals from hidden states to provide a more reliable indicator of SRM reasoning quality for routing.
\rev{Routing also differs in granularity: RouteLLM~\citep{ong2025routellm} and HybridLLM~\citep{ding2024hybridllm} route per \textit{query} and R2R~\citep{fur2r} per \textit{token}, whereas RankGuide routes per \textit{step} using signals aggregated \textit{across} steps, exposing failures invisible to any single token distribution.}

\subsection{Efficient Reasoning via Steering}
Beyond model collaboration, a prominent direction is latent-space steering of LRMs. In steering, hidden representations are modified to guide generation toward desired behaviors. Steering methods \citep{chen2025seal, azizi2025activation, xu2025easysteer} show that intervening in intermediate activations can improve reasoning quality or enforce concise generation patterns without retraining.
While these approaches improve efficiency by enhancing reasoning through latent-space calibration, their performance gains are often limited, and they typically fall short of enabling SRMs to match the accuracy of LRMs.
\rev{A complementary line makes the intervention input-adaptive rather than static~\citep{wang2025sadi, tian2025skipkv}. Our contribution is orthogonal to that axis: we improve the \textit{calibration set} the direction is estimated from rather than how it is applied, so a  rank-filtered calibration set could equally serve a dynamic steering method.}

\begin{figure}[!t]
    \centering
    \includegraphics[width=0.24\linewidth]{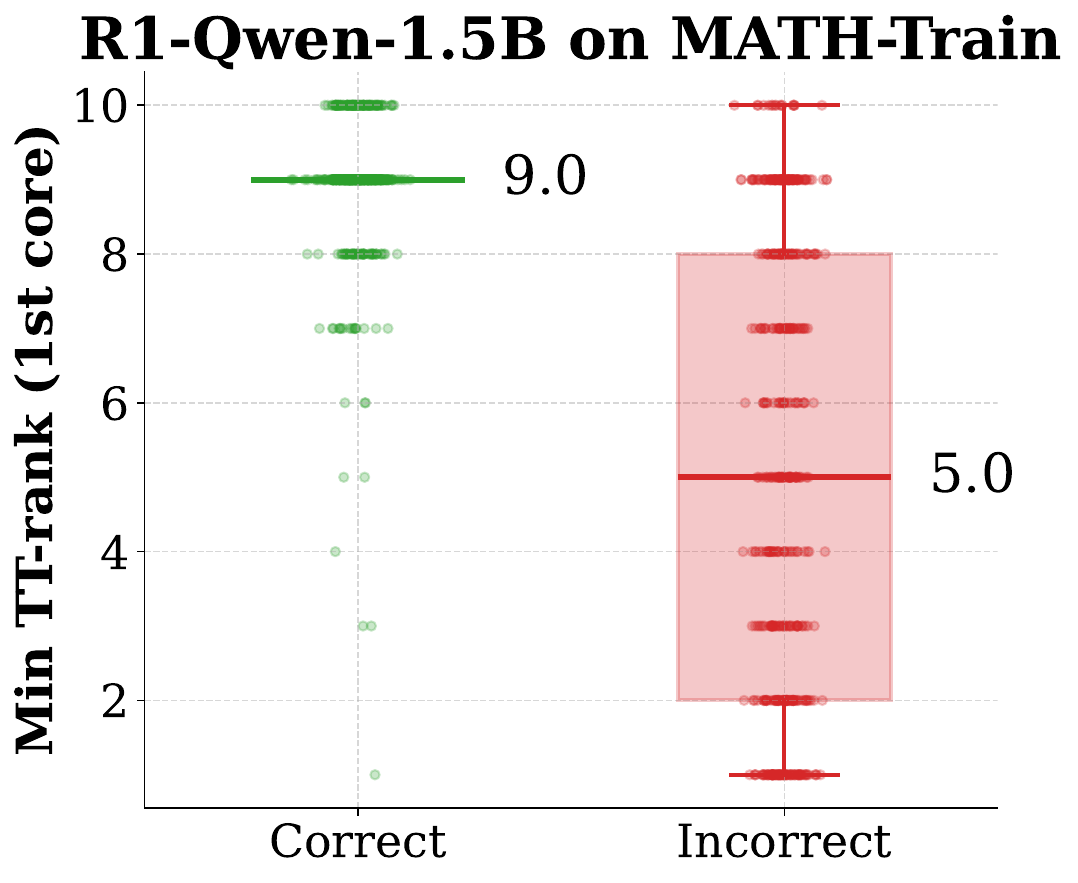}
    \includegraphics[width=0.24\linewidth]{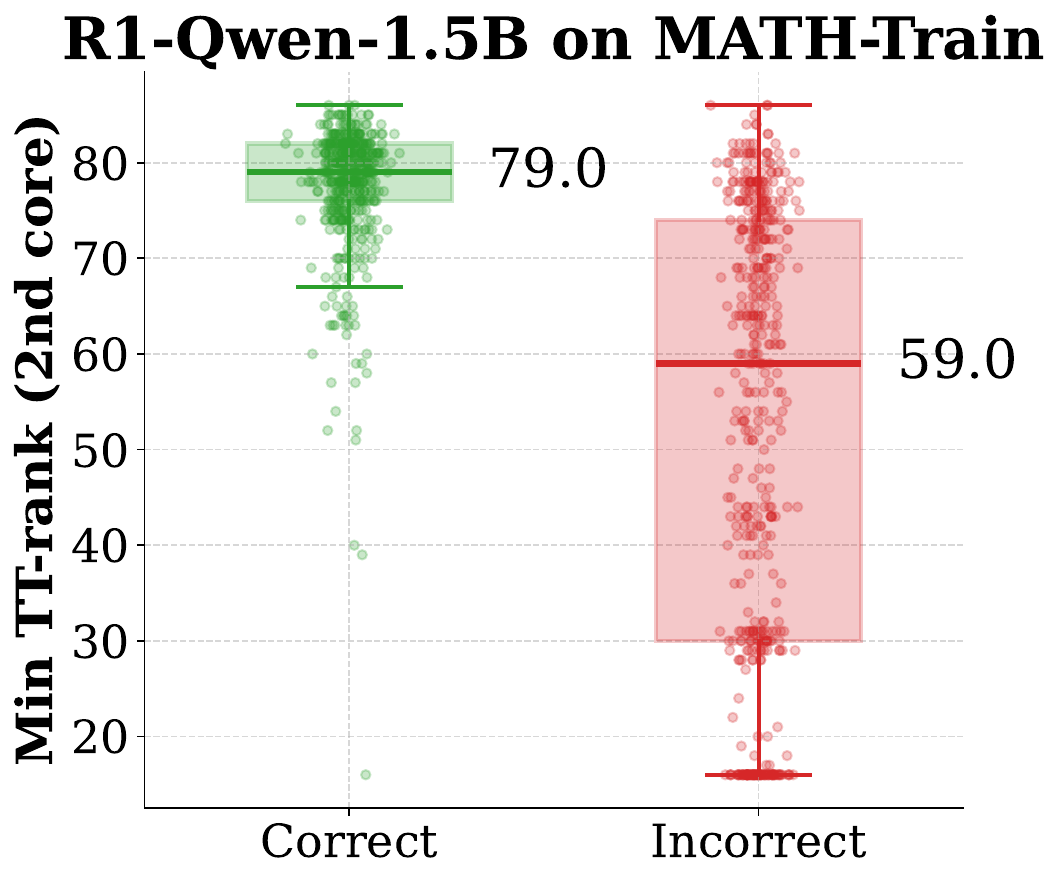}
    \includegraphics[width=0.24\linewidth]{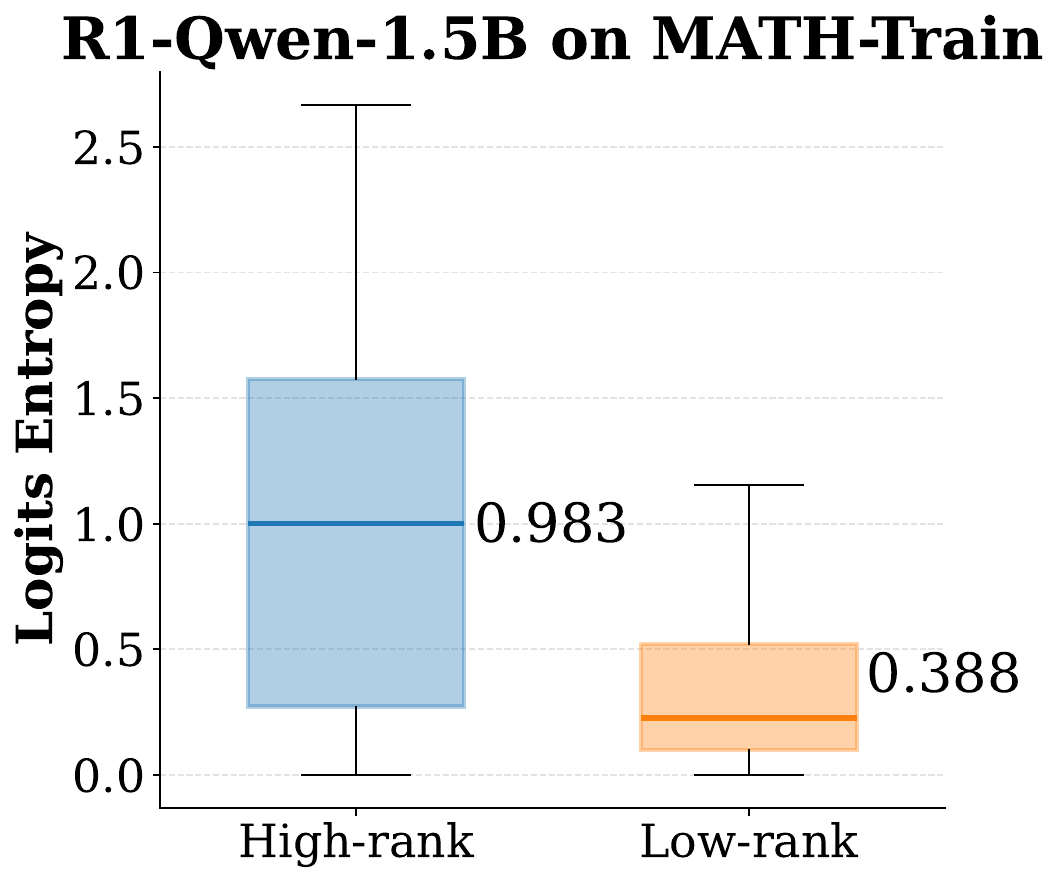}
    \includegraphics[width=0.24\linewidth]{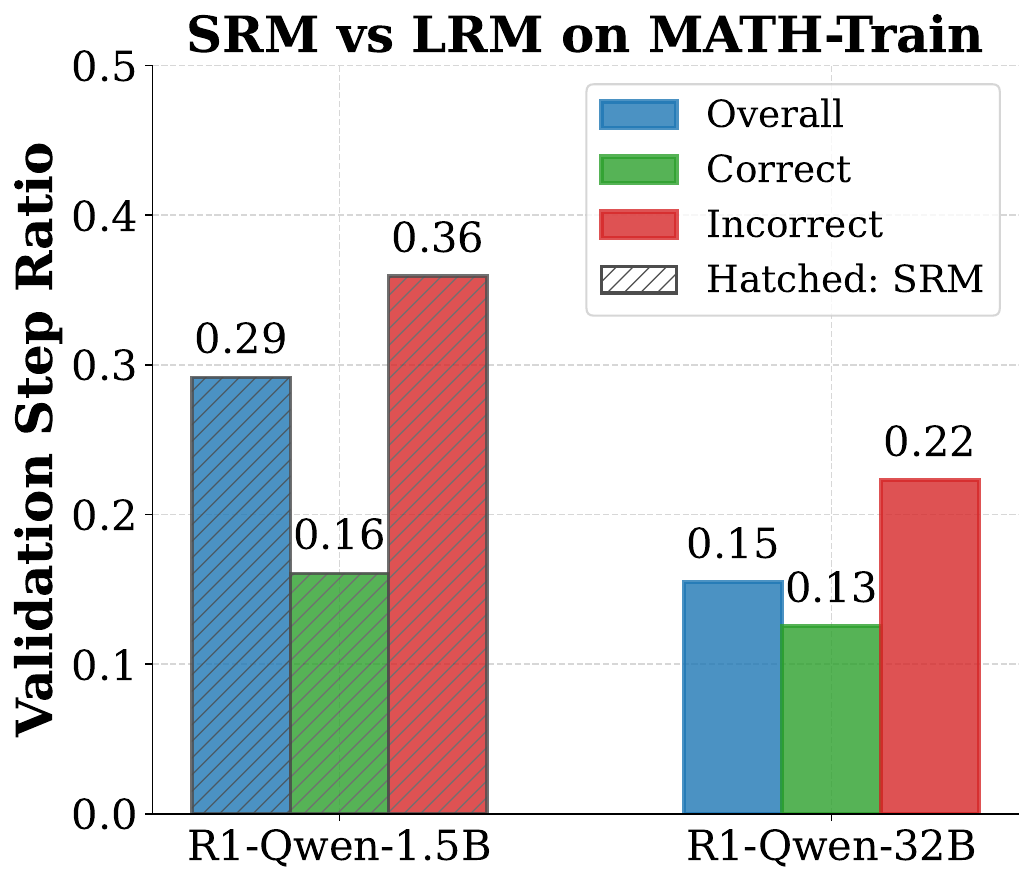}
    \caption{Case study of SRM reasoning dynamics. Left: minimum TT rank along the reasoning-step dimension. Middle-left: minimum TT rank along hidden feature dimensions. Middle-right: first-token entropy of reasoning steps. Right: validation step ratio comparison between SRMs and LRMs.}
    \vspace{-5mm}
    \label{fig:SRM motivation}
\end{figure}

\section{Motivational Case Study}
To better understand the limitations of current SRMs in reasoning and routing, we conduct a series of empirical studies.
To diagnose failure modes and redundancy in SRM generation, we use the delimiter ``\texttt{\textbackslash n\textbackslash n}''
\footnote{\rev{The double newline is a long-standing paragraph and semantic separator in text processing, baked into pretraining corpora, preprocessing pipelines, and chat templates. Modern reasoning models (DeepSeek-R1~\citep{guo2025deepseek}, QwQ, Qwen3~\citep{yang2025qwen3}) emit it between reasoning steps, and prior step-level collaboration work~\citep{panspecreason, zeng2026glimprouter} segments on the same delimiter.}}
to segment model responses into individual reasoning steps, and study both per-step properties and inter-step correlations.


\subsection{Diagnosing SRM Failure through Latent Low-Rank Structure}
\label{sec:motivation1}
To analyze correlations across reasoning steps, we extract middle-layer hidden states at delimiter positions, i.e., where the generated output matches ``\texttt{\textbackslash n\textbackslash n}''.
These representations summarize each reasoning step and influence the subsequent generation.
At generation timestep $t$, we denote the hidden states from the most recent window $W$ of reasoning steps as $\mathbf{H}_{t-W:t}\in\mathbb{R}^{W\times d_{\text{hid}}}$. 
We concatenate these representations in temporal order and reshape them into a higher-order tensor:
\begin{equation}
    \mathcal{H}_{t-W:t} \in \mathbb{R}^{W \times d_1 \times \cdots \times d_N},\ \text{where} \prod_{i=1}^{N} d_i = d_{\text{hid}}.
\end{equation}
Here, $d_i$ denotes the size of the $i$-th mode of the tensor.
We then apply tensor-train (TT) decomposition to capture potential low-rank structures across both the reasoning-step dimension and the hidden feature dimensions.
Formally, the tensor is decomposed as
\begin{equation}
\mathcal{H}_{t-W:t} \approx \mathcal{G}_0 \times \mathcal{G}_1 \times \cdots \times \mathcal{G}_N,
\end{equation}
where the TT-cores satisfy
\begin{equation}
\mathcal{G}_0 \in \mathbb{R}^{W \times r_1}, \quad
\mathcal{G}_k \in \mathbb{R}^{r_{k-1} \times d_k \times r_k}, \; (1 \le k \le N),
\end{equation}
with $\{r_k\}_{k=1}^N$ denoting the TT ranks.
We determine the TT ranks by enforcing a reconstruction error constraint $\epsilon = 0.1$, i.e., 
$\|\mathcal{H}_{t-W:t} - \hat{\mathcal{H}}_{t-W:t}\|_F / \|\mathcal{H}_{t-W:t}\|_F \le \epsilon$, 
where $\hat{\mathcal{H}}_{t-W:t}$ is the TT-reconstructed tensor. 
Specifically, we apply error-bounded TT-SVD that allocates a per-step error budget and selects the smallest ranks satisfying the constraint (see Algorithm~\ref{alg:tt_rank} in the Appendix~\ref{appendix:algorithm}).
The resulting ranks characterize the intrinsic dimensionality of the reasoning dynamics: 
$r_1$ captures the effective rank along the reasoning-step dimension, while the subsequent ranks $\{r_k\}_{k\geq 2}$ reflect the complexity of the hidden feature dimensions. 
In practice, we use a representative hidden-feature rank (e.g., $r_2$) for simplicity.

At each timestep $t$ ($t \ge W$), the sliding-window hidden representations are reshaped into a 4D tensor $\mathcal{H}_{t-W:t} \in \mathbb{R}^{W \times d_1 \times d_2 \times d_3}$, where we set $W=10$, $d_1=16$, $d_2=16$, and $d_3=\frac{d_{hid}}{d_1\times d_2}$ \rev{(further discussion of the factorization setup is in Appendix~\ref{appendix:shape})}. 
We compute TT ranks at each timestep, and report the minimum rank across timesteps for each sample, as shown in Fig.~\ref{fig:SRM motivation}.
Here, we collect generation responses for 500 correct and 500 incorrect samples from the \textsc{MATH-Train} dataset \citep{hendrycks2measuring} using \textsc{DeepSeek-R1-Distill-Qwen-1.5B} for case study. 
As illustrated, incorrect responses consistently exhibit significantly lower minimum ranks than correct ones, both along the reasoning-step dimension (Fig.~\ref{fig:SRM motivation}, left) and the hidden feature dimensions (Fig.~\ref{fig:SRM motivation}, middle-left). 
\textbf{This suggests that erroneous reasoning trajectories tend to collapse into a low-rank latent subspace, indicating a form of generation collapse in SRMs}.

\subsection{Understanding SRM Overconfidence under Latent Low-Rank Collapse}
\label{sec:motivation2}
Next, we analyze the quality of individual reasoning steps. Prior work, such as \textit{GlimpRouter} \citep{zeng2026glimprouter}, suggests that reasoning steps with high-entropy logits indicate model \textit{uncertainty} and should be routed to stronger models for verification or refinement. However, our empirical findings reveal that low-entropy predictions are not necessarily reliable.
Fig.~\ref{fig:SRM motivation} (middle-right) presents the entropy distribution of logits for reasoning steps grouped by their latent-rank structure. We define rank thresholds $T_{r_1}=8$ and $T_{r_2}=60$, 
\rev{each placed between the correct- and incorrect-response medians}, and identify low-rank steps as those satisfying $r_1(\mathcal{H}_{t-W:t}) < T_{r_1} \lor r_2(\mathcal{H}_{t-W:t}) < T_{r_2}$. 
Interestingly, we observe that low-rank steps exhibit substantially lower entropy (0.388 on average) than high-rank steps (0.983 on average), indicating higher confidence. 
Collapsed reasoning trajectories therefore often yield overconfident yet incorrect predictions, so routing strategies relying solely on entropy may fail to identify such errors.
Our findings therefore indicate that \textbf{both high-entropy (uncertain) steps and low-entropy, low-rank (overconfident) steps should be considered as routing candidates in model collaboration systems}.

\subsection{SRM Generation with Heavy Re-Validation}
\label{sec:motivation3}
\rev{A recent survey~\citep{sui2025stopoverthinking} documents an ``overthinking'' problem in chain-of-thought generation, where models produce verbose and redundant outputs.}
We further investigate the generation behavior of SRMs and LRMs and observe a distinct inefficiency: \textbf{SRMs devote substantially more reasoning steps to re-validation than LRMs.}
To quantify this behavior, we measure the \emph{validation step ratio}, defined as the proportion of reasoning steps devoted to checking or verifying intermediate results. Fig.~\ref{fig:SRM motivation} (right) reports this ratio for both SRMs and LRMs on the \textsc{MATH-Train} dataset. As shown, SRMs exhibit significantly higher validation ratios than LRMs. For example, \textsc{R1-Qwen-1.5B} allocates $29\%$ of its reasoning steps to validation on average, exceeding \textsc{R1-Qwen-32B} by approximately $2\times$.

Furthermore, \textbf{we observe that incorrect responses exhibit substantially higher validation ratios than correct ones for both SRMs and LRMs}. Specifically, the validation step ratios for incorrect responses are $2.25\times$ and $1.69\times$ higher than those for correct responses in SRMs and LRMs, respectively. 
This suggests that models tend to repeatedly verify intermediate results without making meaningful progress when the reasoning trajectory is already flawed, increasing generation length without improving correctness, whereas LRMs progress more directly toward correct solutions.

SRMs therefore appear to lack effective self-correction: once reasoning deviates, they rely on repeated verification rather than revising the trajectory. This motivates \emph{steering mechanisms} that intervene in SRM hidden-state dynamics to restore more informative and stable reasoning paths.


\section{RankGuide: Methodology}
\begin{figure}[!t]
    \centering
    \includegraphics[width=\linewidth]{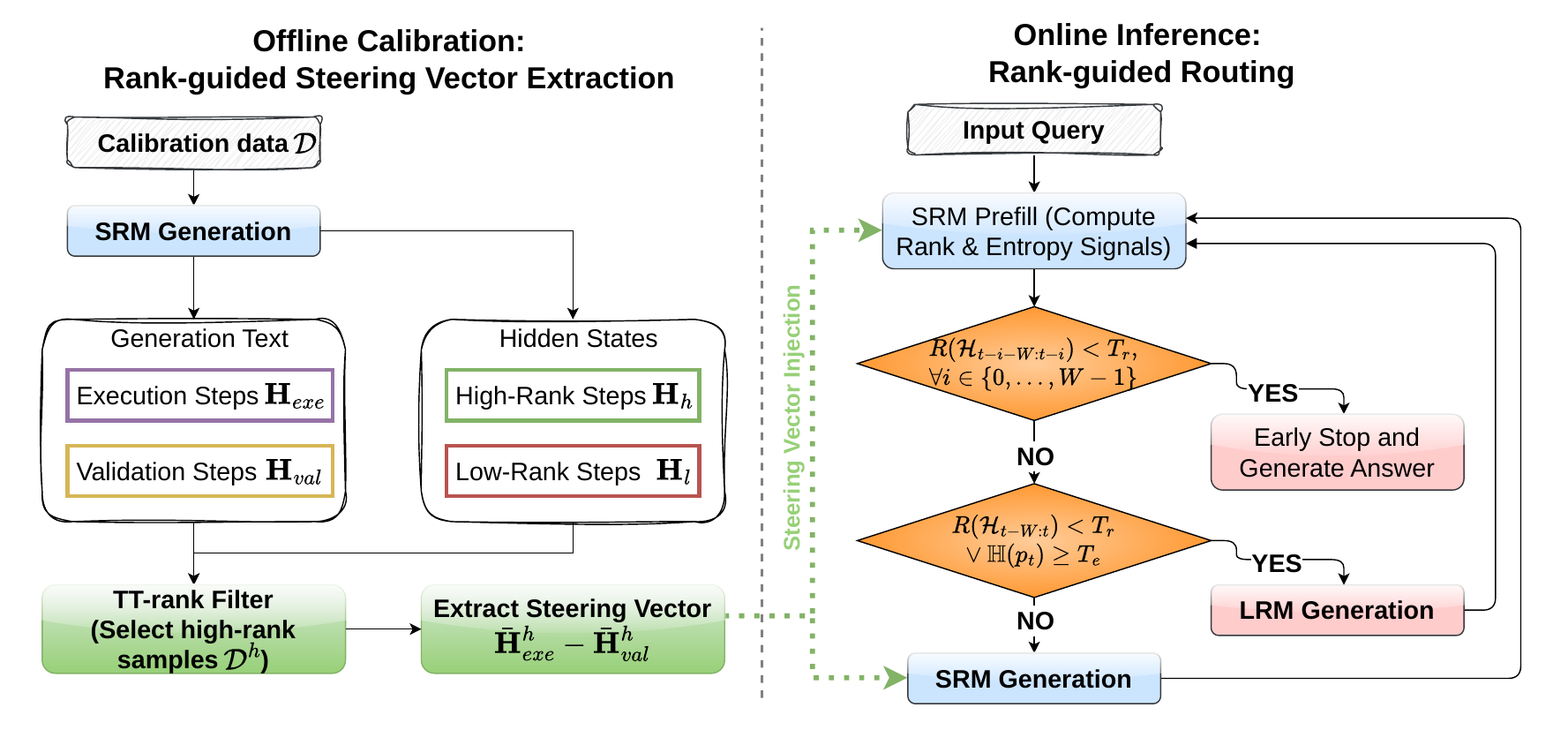}
\vspace{-5mm}
    \caption{Overview of the {RankGuide} framework. 
In the offline calibration phase (left), the method extracts a rank-guided steering vector by contrasting high-rank execution and validation hidden states. 
During online inference (right), the steering vector is injected into SRM generation, where rank and entropy signals are computed to enable adaptive routing.} 
\vspace{-5mm}
    \label{fig:Overview}
\end{figure}
We now present \textbf{RankGuide}, which leverages \textit{tensor-rank signals} to address the SRM failure modes identified above, combining rank-guided routing (\S~\ref{sec:rank_guided_routing}) with rank-guided steering (\S~\ref{sec:rank_guided_steering}) as illustrated in Fig.~\ref{fig:Overview}.

\subsection{Rank-guided Routing}
\label{sec:rank_guided_routing}

\paragraph{Routing Signals.}
Before generating the next step, given the input prompt and previously generated outputs, the SRM performs a prefill pass to extract hidden states $\{\mathbf{H}_t\}_{t=1}^T$ from an intermediate layer, corresponding to reasoning-step boundaries, and collects the top-$k$ output logits for next-token prediction.
At each decoding step $t$, we compute two signals:
(i) a tensor-rank measure $R(\mathcal{H}_{t-W:t})$ (for $t \geq W$), where $\mathcal{H}_{t-W:t}$ denotes the reshaped tensor constructed from a sliding window of $W$ consecutive reasoning steps, capturing dependencies across both the reasoning-step and hidden feature dimensions. Here, $R(\cdot)=\{r_1, r_2\}$, which includes both the reasoning-step rank ($r_1$) and a representative hidden-feature rank ($r_2$), as illustrated in \S~\ref{sec:motivation1};
(ii) the next-token prediction entropy $\mathbb{H}(p_t)$ \rev{over the renormalized top-$k$ distribution $\tilde{p}_t$ at the first token of the upcoming step}, corresponding to \S~\ref{sec:motivation2}.
We introduce separate rank thresholds $T_{r_1}$ and $T_{r_2}$, applied to the reasoning-step rank $r_1$ and hidden-feature rank $r_2$, respectively, along with an \textit{entropy threshold} $T_e$.

\rev{These thresholds are calibration-derived from MATH-Train: each is placed between the correct- and incorrect-response medians, balancing coverage of incorrect steps against unnecessary routing of correct ones.}
As discussed in the motivation section, low-rank states indicate representational collapse, while high entropy reflects model uncertainty. Accordingly, both high-entropy (uncertain) steps and low-entropy, low-rank (overconfident) steps are treated as routing candidates in our framework.

\paragraph{Routing Policy.}
Based on the computed signals, we design a two-stage routing policy to determine whether to continue SRM generation or switch to the LRM. 

During calibration, we observe that when the SRM remains in a persistently low-rank regime across multiple steps, it often produces overconfident yet invalid reasoning and continues generating until reaching the maximum generation length without producing a valid answer. 
Motivated by this observation, we examine whether the low-rank feature persists over a sliding window of $W$ reasoning steps. Specifically, if the tensor-rank measure satisfies
\begin{equation}
R(\mathcal{H}_{t-i-W:t-i}) < T_r,\quad \forall i \in \{0, \dots, W-1\},
\label{eq:early_stop}
\end{equation}
we consider the reasoning trajectory to have fully collapsed and terminate SRM generation to produce the final answer.

Otherwise, we further evaluate both rank and entropy signals at the current step. If
\begin{equation}
R(\mathcal{H}_{t-W:t}) < T_r \;\;\lor\;\; \mathbb{H}(p_t) \ge T_e,
\qquad {\revon \mathbb{H}(p_t) = -\sum\nolimits_{i=1}^{k} \tilde{p}_t^{(i)} \log \tilde{p}_t^{(i)},}
\label{eq:routing}
\end{equation}
we route to the LRM for more reliable generation; otherwise, the SRM continues generation.
This exploits confident, reliable SRM trajectories for efficiency while invoking the LRM exactly when overconfidence or uncertainty is detected.

\subsection{Rank-guided Steering}
\label{sec:rank_guided_steering}

\paragraph{Steering Vector Extraction.}
To reduce redundant revalidation in SRM generation, we extract a rank-guided steering vector at an offline calibration stage, as shown in Fig.~\ref{fig:Overview} (left).

Given a calibration dataset $\mathcal{D}$, we categorize reasoning steps along two complementary dimensions. 
First, inspired by SEAL~\citep{chen2025seal}, we classify steps in the generated text space into \emph{execution} and \emph{validation} types, where execution steps correspond to solving the problem and validation steps correspond to verifying intermediate or final results. 
Specifically, we identify \emph{validation} steps using heuristic keywords (e.g., ``Alternatively'', ``Wait'', ``verify'', etc.). 

Second, we use the tensor-rank signal defined in \S~\ref{sec:motivation1} to classify each reasoning step into high-rank and low-rank groups based on the TT ranks of its associated hidden representations. 
This captures the intrinsic dimensionality of the reasoning process and serves as an indicator of representational richness versus collapse. 
Formally, given the calibration set $\mathcal{D}$, we define the sets of execution and validation steps as
\begin{equation}
\mathcal{S}_{\text{exe}}^{\mathcal{D}}
=
\{\, j \mid j \in \text{exe},\ x(j) \in \mathcal{D} \,\},
\qquad
\mathcal{S}_{\text{val}}^{\mathcal{D}}
=
\{\, j \mid j \in \text{val},\ x(j) \in \mathcal{D} \,\},
\end{equation}
where $j$ indexes reasoning steps and $x(j)$ denotes the sample to which step $j$ belongs. 
Based on our observation that low-rank calibration samples can degrade the steering vector by introducing collapsed representations, we apply a tensor-rank filter to the calibration dataset $\mathcal{D}$ and retain only high-rank samples $\mathcal{D}^{\mathrm{h}} \subseteq \mathcal{D}$ that do not contain low-rank steps. 
We then compute the category-wise averages over the filtered sets and extract the enhanced steering vector as,
\begin{equation}
\mathbf{v}_{\text{steer}}^h =
\bar{\mathbf{H}}_{\text{exe}}^h -
\bar{\mathbf{H}}_{\text{val}}^h,
\ \text{where}\ 
\bar{\mathbf{H}}_{\text{exe}}^h =
\frac{1}{|\mathcal{S}_{\text{exe}}^{\mathcal{D}^{\mathrm{h}}}|}
\sum_{j \in \mathcal{S}_{\text{exe}}^{\mathcal{D}^{\mathrm{h}}}}
\mathbf{H}_{j},
\qquad
\bar{\mathbf{H}}_{\text{val}}^h =
\frac{1}{|\mathcal{S}_{\text{val}}^{\mathcal{D}^{\mathrm{h}}}|}
\sum_{j \in \mathcal{S}_{\text{val}}^{\mathcal{D}^{\mathrm{h}}}}
\mathbf{H}_{j}.
\end{equation}
\paragraph{Steering Vector Injection.}
During inference, we inject the rank-guided steering vector into the SRM to modulate its hidden-state trajectory. 
To preserve token-level fluency while influencing high-level reasoning, this intervention is applied selectively at reasoning step boundaries. 
Concretely, at the end of each generation step, we perform an additive intervention on the intermediate hidden representation:
\begin{equation}
\mathbf{H}_t \leftarrow \mathbf{H}_t + \alpha \cdot \mathbf{v}_{\text{steer}}^h,
\quad \text{if } y_t \in \mathcal{T}_{\mathrm{seg}}.
\end{equation}
Here, $\alpha$ is a scaling coefficient controlling the strength of the intervention, $y_t$ denotes the token generated at decoding step $t$, and 
$\mathcal{T}_{\mathrm{seg}}$ denotes the set of segment delimiter tokens marking reasoning step boundaries, which is instantiated as ``\texttt{\textbackslash n\textbackslash n}'' for implementation simplicity.
This operation biases the model's representations toward informative high-rank execution patterns, promoting more efficient and effective reasoning dynamics.
As a result, the SRM produces fewer redundant validation or collapsed reasoning steps, thereby improving the effectiveness of the downstream routing mechanism.

\begin{figure}[!t]
    \centering
    \includegraphics[width=\linewidth]{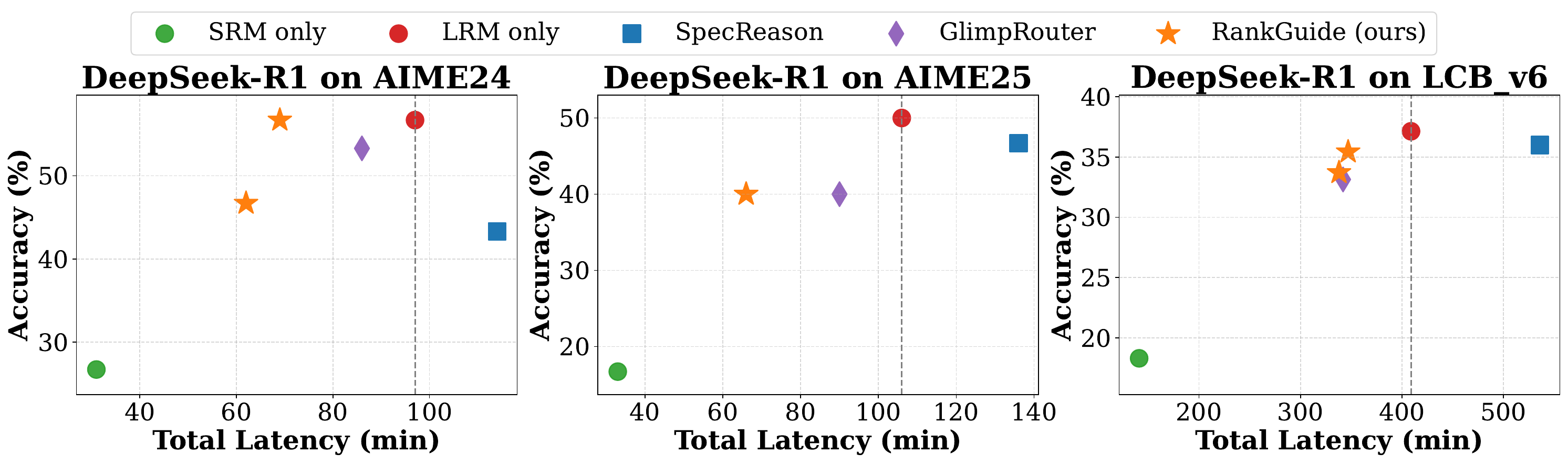}
    \includegraphics[width=\linewidth]{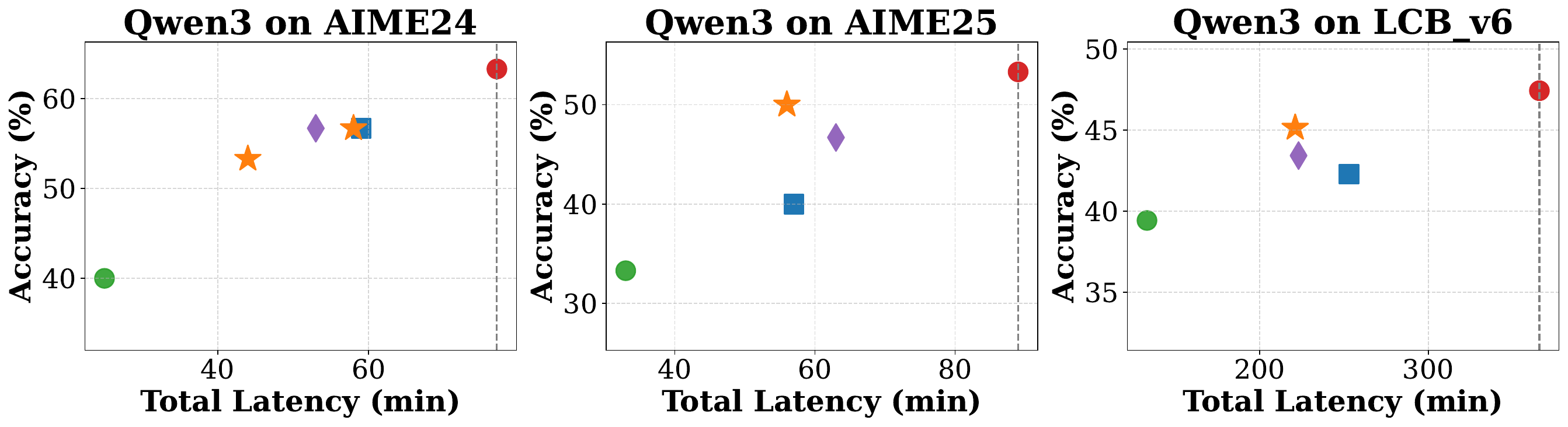}
\vspace{-5mm}
    \caption{Accuracy vs. reasoning latency on AIME-24, AIME-25, LCB-v6 using DeepSeek-R1 (Qwen-1.5B / 32B, top) and Qwen3 (4B / 32B, bottom). We report two operating points for our method (best-accuracy and best-latency), which may overlap in some cases.}
\vspace{-5mm}
    \label{fig:main-results}
\end{figure}

\section{Experiments}
\subsection{Experimental Setup}
\textbf{Models.}
We evaluate our approach on models from the Qwen3~\citep{yang2025qwen3} and DeepSeek-R1~\citep{guo2025deepseek} families. 
Specifically, we consider two SRM–LRM pairs: Qwen3-4B / 32B, and DeepSeek-R1-Distill-Qwen-1.5B / 32B.
These pairs span different model scales and architectures.
\rev{For the Qwen3 pair we use the \texttt{Qwen3-4B} checkpoint as the SRM, following the GlimpRouter paper.}\footnote{\rev{The public GlimpRouter code instead uses \texttt{Qwen3-4B-Thinking-2507}, which their paper does not state. We follow the paper, using \texttt{Qwen3-4B} and re-running their released code, which accounts for the difference between our reproduced SRM-only/GlimpRouter numbers on Qwen3 and theirs.}}

\textbf{Datasets.}
We evaluate our method on challenging reasoning benchmarks \rev{spanning three domains}.
For mathematical reasoning, we use AIME24~\citep{aime24} and AIME25~\citep{aime25}, which consist of high-difficulty competition problems.
For code generation, we use LiveCodeBench v6~\citep{jain2025livecodebench}, a benchmark designed to evaluate performance on competitive programming tasks.
\rev{For scientific reasoning, we use GPQA-Diamond~\citep{rein2024gpqa}, the highest-quality subset of GPQA: 198 graduate-level questions in biology, physics and chemistry.}

\textbf{Baselines.}
We compare our method against multiple representative baselines. 
\textit{SRM-only} and \textit{LRM-only} use a single model without routing, bounding latency and accuracy respectively.
We also include two recent routing-based methods: \textit{SpecReason}~\citep{panspecreason}, which performs speculative reasoning with LLM-based verification and rollback, and \textit{GlimpRouter}~\citep{zeng2026glimprouter}, which adopts an entropy-based routing strategy for adaptive model selection without rollback.
Implementation details and hyperparameter settings are provided in Appendix~\ref{appendix:setup}, \rev{and the routing overhead is quantified in Appendix~\ref{appendix:overhead}}.


\begin{table}[t]
\centering
\caption{Performance comparison on AIME24, LiveCodeBench v6 \rev{and GPQA-Diamond} with DeepSeek-R1 models. We report accuracy (pass@1), latency, \rev{average} number of steps, and validation step ratio. Accuracy and latency are reported relative to LRM, while the number of steps is reported relative to SRM. 
}
\resizebox{\linewidth}{!}{
\begin{tabular}{lcccc}
\toprule
\textbf{Method} & \textbf{Accuracy (\%)} & \textbf{Latency} & \textbf{Avg. Steps} & \textbf{Validation Step Ratio (\%)} \\
\midrule
\multicolumn{5}{c}{\textbf{DeepSeek-R1 (1.5B/32B) on AIME24}} \\
\midrule
SRM & 26.7 & 31 & 295 & 24.70 \\
LRM & \textbf{56.7} & 97 & 268 (-10.1\%) & 19.50 (-5.2\%) \\
\midrule
SpecReason~\citep{panspecreason}  & 43.3 (-13.4) & 114 (0.85$\times$) & 261 (-13.1\%) & 17.21 (-7.5\%) \\
GlimpRouter~\citep{zeng2026glimprouter}  & 53.3 (-3.3) & 86 (1.13$\times$) & 325 (+9.1\%) & 19.93 (-4.8\%) \\
\rowcolor{orange!15}\textbf{RankGuide (ours)} & \textbf{56.7 (+0.0)} & \textbf{69 (1.41$\times$)} & \textbf{247 (-19.3\%)} & \textbf{15.68 (-9.0\%)} \\
\midrule
\multicolumn{5}{c}{\textbf{DeepSeek-R1 (1.5B/32B) on LiveCodeBench v6}} \\
\midrule
SRM & 18.29 & 141 & 222 & 36.50 \\
LRM & \textbf{37.14} & 409 & 234 (+5.4\%) & 23.20 (-36.4\%) \\
\midrule
SpecReason~\citep{panspecreason}  & 36.00 (-1.14) & 536 (0.76$\times$) & 239 (+7.2\%) & 23.17 (-36.5\%) \\
GlimpRouter~\citep{zeng2026glimprouter}  & 33.14 (-4.00) & \textbf{342 (1.20$\times$)} & 236 (+6.3\%) & 26.30 (-28.0\%) \\
\rowcolor{orange!15}\textbf{RankGuide (ours)} & 35.43 (-1.71) & 347 (1.18$\times$) & \textbf{196 (-16.2\%)} & \textbf{21.02 (-42.4\%)} \\
\midrule
\multicolumn{5}{c}{\rev{\textbf{DeepSeek-R1 (1.5B/32B) on GPQA-Diamond}}} \\
\midrule
\rev{SRM} & \rev{41.4 (-20.2)} & \rev{152} & \rev{133} & \rev{45.46} \\
\rev{LRM} & \rev{\textbf{61.6}} & \rev{387} & \rev{112 (-15.8\%)} & \rev{\textbf{31.88 (-13.58\%)}} \\
\midrule
\rev{GlimpRouter~\citep{zeng2026glimprouter}} & \rev{55.1 (-6.5)} & \rev{\textbf{315 (1.23$\times$)}} & \rev{121 (-9.0\%)} & \rev{34.73 (-10.73\%)} \\
\rowcolor{orange!15}\rev{\textbf{RankGuide (ours)}} & \rev{57.6 (-4.0)} & \rev{\textbf{315 (1.23$\times$)}} & \rev{\textbf{114 (-14.3\%)}} & \rev{32.75 (-12.71\%)} \\
\bottomrule
\end{tabular}
}
\label{tab:main_results}
\end{table}

\subsection{Main Results}
\textbf{RankGuide consistently achieves the highest accuracy with improved latency trade-offs.}
Figure~\ref{fig:main-results} shows the accuracy–latency trade-offs under various threshold settings \rev{(full sweep in Appendix~\ref{appendix:routing-thresh})}: RankGuide consistently outperforms existing routing approaches in accuracy while remaining faster than the LRM.
For example, on AIME-24 with DeepSeek-R1, RankGuide improves accuracy by $3.3\%$ while reducing latency by $1.25\times$ compared to GlimpRouter, and achieves comparable accuracy to LRM with $1.41\times$ lower latency. 
Similarly, on AIME-25 with Qwen3, RankGuide improves accuracy by $3.3\%$ while reducing latency by $1.13\times$ compared to GlimpRouter, and attains comparable accuracy to LRM with $1.59\times$ lower latency. 
In contrast, SpecReason's repeated verification and rollback often make it slower than the LRM at lower accuracy, while GlimpRouter improves efficiency but cannot correct certain reasoning failures.

\textbf{RankGuide reduces reasoning steps and re-validation redundancy in SRMs.}
Table~\ref{tab:main_results} shows that these improvements are driven by more efficient reasoning behavior.
RankGuide reduces the total number of steps by 19.3\% on AIME24 and 16.2\% on LiveCodeBench v6, while lowering the validation step ratio by up to 42.4\% compared to SRM and up to 14.4\% compared to previous routing methods.
In contrast, GlimpRouter, which relies solely on uncertainty-based routing, fails to eliminate over-confident yet incorrect reasoning steps and thus maintains higher validation ratios and unnecessary reasoning steps.
\rev{On GPQA-Diamond, which is $6.6\times$ larger than AIME and outside our calibration domain, RankGuide beats GlimpRouter by $2.5$ points at matched latency with $5.8\%$ fewer steps and $4.1\%$ fewer tokens, recovering $93.5\%$ of LRM accuracy at $1.23\times$ lower latency. It also compresses the SRM's indiscriminate validation profile ($45.46\%$) to the LRM's selective level ($32.75\%$ vs.\ $31.88\%$).
A more detailed breakdown of validation steps across correct and incorrect responses is given in Appendix~\ref{appendix:revalidation}.}

\subsection{Ablation Studies}

\begin{table}[t]
\centering
\caption{Ablation of RankGuide on AIME24 using DeepSeek-R1 and Qwen3 model pairs. Progressive incorporation of rank-guided steering and routing improves accuracy while reducing reasoning steps and latency compared to entropy-based routing baselines.}
\setlength{\tabcolsep}{4pt}
\resizebox{\linewidth}{!}{
\begin{tabular}{l l c c c}
\toprule
\textbf{Method} & \textbf{Thresholds} & \textbf{Acc (\%)} & \textbf{Latency (min)} & \textbf{Avg. steps} \\
\midrule
\multicolumn{5}{c}{\textbf{DeepSeek-R1 (1.5B / 32B) on AIME24}} \\
\midrule
GlimpRouter~\citep{zeng2026glimprouter} & \rev{$T_e=0.9$} & 46.7 & 81 & 305 \\
+ Rank-guided Steering & \rev{$T_e=0.9$} & 50.0 (+3.3) & 74 (1.09$\times$) & 278 (-8.9\%) \\
\rowcolor{orange!15} + Rank-guided Routing &  \rev{$T_e=0.9, T_{r_1}=7, T_{r_2}=60$} & \textbf{56.7 (+10.0)} & 69 (1.17$\times$) & 246 (-19.3\%) \\
\rowcolor{orange!15} &  \rev{$T_e=0.9, T_{r_1}=8, T_{r_2}=60$} & 53.3 (+6.6) & \textbf{67 (1.21$\times$)} & \textbf{236 (-22.6\%)} \\
\midrule
\multicolumn{5}{c}{\textbf{Qwen3 (4B / 32B) on AIME24}} \\
\midrule
GlimpRouter~\citep{zeng2026glimprouter} & \rev{$T_e=1.1$} & 50.0 & 53 & 207 \\
+ Rank-guided Steering & \rev{$T_e=1.1$} & 50.0 (+0.0) & 45 (1.18$\times$) & 171 (-17.4\%) \\
\rowcolor{orange!15} + Rank-guided Routing & \rev{$T_e=1.1,\; T_{r_1}=7, T_{r_2}=50$} & \textbf{53.3 (+3.3)} & 44 (1.20$\times$) & 194 (-6.3\%) \\
\rowcolor{orange!15} & \rev{$T_e=1.1,\; T_{r_1}=7, T_{r_2}=60$} & 50.0 (+0.0) & \textbf{42 (1.26$\times$)} & \textbf{180 (-13.0\%)} \\
\bottomrule
\end{tabular}
}
\label{tab:rankguide-abla}
\end{table}
\subsubsection{Ablation Study on RankGuide Components}
From Table~\ref{tab:rankguide-abla}, we observe that progressively integrating rank-guided steering and routing yields consistent improvements in both reasoning efficiency and accuracy on the AIME24 benchmark across different model families. 
Introducing rank-guided steering primarily reduces the average number of reasoning steps, while maintaining or slightly improving accuracy. 
Building upon this, incorporating rank-guided routing further improves accuracy by explicitly identifying over-confident reasoning steps and routing them to stronger LRMs, while still reducing overall reasoning length and improving latency. 
For example, on DeepSeek-R1, RankGuide improves accuracy by up to +10.0\% while reducing average reasoning steps by 22.6\% and achieving $1.21\times$ faster inference compared to the GlimpRouter baseline.

\begin{table}[t]
\centering
\revon  
\caption{Ablation of RankGuide on AIME25 with DeepSeek-R1 (1.5B/32B), with and without collapsed-trajectory early stopping (ES).}
\setlength{\tabcolsep}{4pt}
\resizebox{\linewidth}{!}{
\begin{tabular}{llcccc}
\toprule
\textbf{Method} & \textbf{Config} & \textbf{Acc (\%)} & \textbf{Latency (min)} & \textbf{Avg. Tokens} & \textbf{ES fire rate} \\
\midrule
GlimpRouter~\citep{zeng2026glimprouter} & $T_e=1.1$ & 40.0 & 90 & 10640 & 0.000 \\
\midrule
RankGuide w/o ES & $T_e=0.9$, $(T_{r_1},T_{r_2})=(7,70)$ & \textbf{43.3} & 87 (-3.3\%) & 9555 (-10.2\%) & 0.000 \\
RankGuide w/o ES & $T_e=0.9$, $(T_{r_1},T_{r_2})=(7,80)$ & 40.0 & 87 (-3.3\%) & 9530 (-10.4\%) & 0.000 \\
\rowcolor{orange!15}RankGuide & $T_e=0.9$, $(T_{r_1},T_{r_2})=(7,70)$ & 40.0 & 84 (-6.7\%) & 9196 (-13.6\%) & 0.100 \\
\rowcolor{orange!15}RankGuide & $T_e=0.9$, $(T_{r_1},T_{r_2})=(7,80)$ & 40.0 & \textbf{80 (-11.1\%)} & \textbf{8522 (-19.9\%)} & 0.267 \\
\bottomrule
\end{tabular}
}
\label{tab:latency-decomp}
\end{table}

{\revon
Table~\ref{tab:latency-decomp} separates the two mechanisms in rank-guided routing: step-level routing and early stopping.
We define early-stopping (ES) fire rate as the fraction of problems terminated by the persistent-low-rank condition of Eq.~\ref{eq:early_stop}.
Overall, RankGuide without early stopping (with only step-level routing and steering) accounts for a ${\approx}3\%$ latency and ${\approx}10\%$ token reduction relative to the entropy-only baseline, by suppressing redundant self-validation and avoiding LRM calls on steps where the SRM already reasons well.
Collapsed-trajectory early stopping contributes a further $3$--$7$ minutes, scaling monotonically with its fire rate ($3$ min at $0.100$, $7$ min at $0.267$), evidence that the saving comes from truncating degenerate trajectories.
Accuracy is preserved at $40.0\%$ across both early-stopping configurations, giving a cumulative $11.1\%$ latency reduction at $19.9\%$ fewer tokens.
Appendix~\ref{appendix:rank-thresh} reports how often early stopping fires and the accuracy and latency trends of our RankGuide across a broader threshold sweep.
\par}

\begin{table}[t]
\centering
\caption{Ablation of rank-guided steering on AIME 24 and AIME 25 using DeepSeek-R1-Qwen-1.5B. We report accuracy, total token length, and the number of calibration samples used to generate the steering vector.}
\setlength{\tabcolsep}{6pt}
\resizebox{\linewidth}{!}{
\begin{tabular}{lccc|ccc}
\toprule
\textbf{Method} 
& \multicolumn{3}{c}{\textbf{AIME 24}} 
& \multicolumn{3}{c}{\textbf{AIME 25}} \\
\cmidrule(lr){2-4} \cmidrule(lr){5-7}
& Acc (\%) & \#Tokens & \#Samples 
& Acc (\%) & \#Tokens & \#Samples \\
\midrule
SRM only 
& 26.7 (0.0) & 14368 (0.0\%) & / 
& 16.7 (0.0) & 16475 (0.0\%) & / \\
+ SEAL Steering 
& 33.3 (+6.6) & 14636 (+1.9\%) & 1000 
& 13.3 (-3.4) & 15336 (-6.9\%) & 1000 \\
\rowcolor{orange!15} + RankGuide Steering 
& \textbf{33.3 (+6.6)} & \textbf{13973 (-2.7\%)} & \textbf{738} 
& \textbf{16.7 (0.0)} & \textbf{15082 (-8.5\%)} & \textbf{632} \\
\bottomrule
\end{tabular}
}
\label{tab:steering_aime}
\end{table}
\subsubsection{Ablation Study on Tensor-Rank-Guided Steering}
We evaluate the effectiveness of tensor-rank-guided steering by comparing it with the SRM-only baseline and the prior SEAL steering method on AIME 24 and AIME 25 using DeepSeek-R1-Qwen-1.5B.
\rev{Steering targets token reduction at non-degraded accuracy.}
As shown in Tab.~\ref{tab:steering_aime}, RankGuide \rev{does this on both datasets: $+6.6$ accuracy points at $2.7\%$ fewer tokens on AIME24, and accuracy held at the SRM baseline ($16.7\%$) at $8.5\%$ fewer tokens on AIME25.}
In contrast, SEAL exhibits unstable behavior, leading to either decreased accuracy \rev{($-3.4\%$)} or increased token length \rev{($+1.9\%$)}.
For SEAL, we construct the steering vector using 1000 samples from the MATH-Train dataset, following the original setup. For RankGuide, we apply the TT-rank-based calibration sample filtering described in \S~\ref{sec:rank_guided_steering} to obtain a higher-quality subset for steering vector construction. We report the number of remaining calibration samples after filtering in the table, where less informative or harmful samples are removed.
Overall, tensor-rank-guided steering is a more reliable and efficient alternative to prior steering methods, improving both reasoning quality and generation efficiency.

\begin{table}[t]
\centering
\revon  
\caption{Effect of rank-guided steering on routing frequency, evaluated on AIME25 with DeepSeek-R1 1.5B/32B (without steering $\rightarrow$ with steering).}
\setlength{\tabcolsep}{6pt}
\begin{tabular}{lcccc}
\toprule
$(T_{r_1}, T_{r_2}, T_e)$ & \textbf{Acc (\%)} & $\Delta$ \textbf{LRM-routing frac.} & $\Delta$ \textbf{Latency} & $\Delta$ \textbf{Steps} \\
\midrule
$(7, 80, 0.9)$ & $40.0 \rightarrow 40.0$ & $\mathbf{-1.6\%}$ & $-10.0\%$ & $-12.0\%$ \\
$(7, 70, 0.9)$ & $40.0 \rightarrow 40.0$ & $\mathbf{-5.4}\%$ & $-6.7\%$ & $-3.6\%$ \\
\bottomrule
\end{tabular}
\vspace{-2mm}
\label{tab:steer-routing}
\end{table}

{\revon
Rank-guided steering is calibrated offline to counteract representational collapse, while rank-guided routing escalates exactly the steps where collapse is detected during decoding.
Effective steering should therefore \textit{reduce} the number of low-rank routing triggers.
Table~\ref{tab:steer-routing} tests this by holding $(T_{r_1}, T_{r_2}, T_e)$ and early stopping fixed and toggling only the steering vector: accuracy is unchanged while latency and the escalated fraction both fall, the latter by $1.6\%$ and $5.4\%$.
Steering thus mitigates collapse before routing must intervene, leaving routing to handle residual uncertainty and overconfidence.
\par}

\section{Conclusion}
We analyze the failure modes of small reasoning models (SRMs) in collaborative systems -- uncertainty, overconfidence, and heavy revalidation -- and find that the tensor-train rank of consecutive hidden states exposes failures that output entropy alone misses.
Building on this, we propose \textbf{RankGuide}, a training-free framework that uses the same signal twice: offline, to filter the calibration set from which we extract a steering vector that trims redundant generation, and online, to route unreliable steps to the LRM and terminate collapsed trajectories.
\rev{Across mathematics, code generation and scientific QA, RankGuide cuts latency by up to $1.75\times$ and $1.36\times$ over LRM-only inference and state-of-the-art step-level routing methods, respectively, at competitive accuracy.}



\section*{Acknowledgments}
This work is co-funded by Intel Strategic Research Sectors (SRS) - Systems Integration
SRS \& Devices SRS.


\bibliography{colm2026_conference}
\bibliographystyle{colm2026_conference}

\newpage
\appendix
\section{Appendix}
\subsection{Error-bounded TT Rank Selection Algorithm}
\label{appendix:algorithm}

\paragraph{Explanation.}
Algorithm~\ref{alg:tt_rank} applies TT-SVD~\citep{oseledets2011tensor} with error-bounded rank selection to determine TT ranks for the hidden-state tensor $\mathcal{H}_{t-W:t}$. The procedure allocates the total error budget uniformly across the $N$ decomposition steps and selects, at each step, the smallest rank whose discarded singular-value energy is within the per-step threshold. The resulting ranks $\{r_k\}_{k=1}^N$ characterize the intrinsic low-rank structure of the tensor: $r_1$ reflects the effective rank along the reasoning-step dimension, while $\{r_k\}_{k\geq 2}$ characterize dependencies in the hidden feature dimensions.
We next formalize the TT reconstruction used in the error analysis and show that the proposed algorithm satisfies the desired error bound.

\begin{algorithm}[t]
\caption{Error-bounded TT-SVD Rank Selection}
\label{alg:tt_rank}
\begin{algorithmic}[1]
\REQUIRE Hidden states tensor $\mathcal{H}_{t-W:t} \in \mathbb{R}^{W \times d_1 \times \cdots \times d_N}$, target error $\epsilon$
\ENSURE TT ranks $\{r_k\}_{k=1}^N$, TT cores $\{\mathcal{G}_k\}_{k=0}^N$

\STATE Compute tensor Frobenius norm $\|\mathcal{H}_{t-W:t}\|_F^2$
\STATE Set per-step error budget:
\[
\delta^2 \leftarrow \frac{\epsilon^2 \|\mathcal{H}_{t-W:t}\|_F^2}{N}
\]
\STATE Initialize $r_0 = 1$, unfolding tensor $\mathcal{H}_{t-W:t}' \leftarrow \mathcal{H}_{t-W:t}$
\FOR{$k = 1$ to $N$}
    \STATE Reshape unfolding:
    \[
    \mathbf{H}' \leftarrow \mathrm{reshape}(\mathcal{H}_{t-W:t}', (r_{k-1} d_k, -1))
    \]
    \STATE Compute SVD:
    \[
    \mathbf{H}' = \mathbf{U} \mathbf{\Sigma} \mathbf{V}^\top
    \]
    \STATE Select smallest $r_k$ such that:
    \[
    \sum_{i > r_k} \sigma_i^2 \le \delta^2
    \]
    \STATE Truncate $\mathbf{U}, \mathbf{\Sigma}, \mathbf{V}$ to rank $r_k$
    \STATE Update $k$-th tensor core and unfolding:
    \[
    \mathcal{G}_{k-1}\leftarrow \mathbf{U}_r, \mathbf{H}' \leftarrow \mathbf{\Sigma}_r \mathbf{V}_r^\top
    \]
\ENDFOR
\STATE  Update last tensor core: $\mathcal{G}_{N}\leftarrow \mathbf{H}'$
\RETURN $\{r_k\}_{k=1}^N$, $\{\mathcal{G}_k\}_{k=0}^N$
\end{algorithmic}
\end{algorithm}

\begin{definition}[TT Reconstruction]
Given TT cores $\{\mathcal{G}_k\}_{k=0}^{N}$ obtained from $\mathcal{H}_{t-W:t}$, where
\[
\mathcal{G}_0 \in \mathbb{R}^{W \times r_1}, \quad
\mathcal{G}_k \in \mathbb{R}^{r_{k-1} \times d_k \times r_k} \ (1 \le k \le N),
\]
the reconstructed tensor $\hat{\mathcal{H}}_{t-W:t}$ is obtained via sequential contractions:
\begin{equation}
\hat{\mathcal{H}}_{t-W:t}
= \mathcal{G}_0 \times_3^1 \mathcal{G}_1 \times_3^1 \cdots \times_3^1 \mathcal{G}_N,
\label{eqa:reconstruction}
\end{equation}
where $\times_3^1$ denotes contraction between the third mode of $\mathcal{G}_{k-1}$ and the first mode of $\mathcal{G}_k$, corresponding to the shared rank dimension $r_k$.
\end{definition}

\begin{theorem}
Let $\hat{\mathcal{H}}_{t-W:t}$ be the TT approximation constructed as in Eq.~\eqref{eqa:reconstruction}, with TT cores obtained from Algorithm~\ref{alg:tt_rank}. Then the reconstruction error satisfies
\[
\frac{\|\mathcal{H}_{t-W:t}-\hat{\mathcal{H}}_{t-W:t}\|_F}{\|\mathcal{H}_{t-W:t}\|_F} \le \epsilon.
\]
\end{theorem}

\begin{proof}
At step $k$, let $\mathcal{E}_k$ denote the truncation error induced by discarding singular values beyond rank $r_k$. By construction of Algorithm~\ref{alg:tt_rank}, the selected rank satisfies
\[
\|\mathcal{E}_k\|_F^2 = \sum_{i>r_k}\sigma_i^2 \le \delta^2,
\]
where
\[
\delta^2 = \frac{\epsilon^2 \|\mathcal{H}_{t-W:t}\|_F^2}{N}.
\]

TT-SVD is a sequence of orthogonal projections, so the truncation errors from different steps are orthogonal. Therefore, the total squared reconstruction error is bounded by the sum of the per-step squared errors:
\[
\|\mathcal{H}_{t-W:t}-\hat{\mathcal{H}_{t-W:t}}\|_F^2
\le \sum_{k=1}^{N}\|\mathcal{E}_k\|_F^2
\le N\delta^2
= \epsilon^2 \|\mathcal{H}_{t-W:t}\|_F^2.
\]
Taking square roots on both sides gives
\[
\frac{\|\mathcal{H}_{t-W:t}-\hat{\mathcal{H}_{t-W:t}}\|_F}{\|\mathcal{H}_{t-W:t}\|_F} \le \epsilon.
\]
\end{proof}

\begin{figure}[!t]
    \centering
    \includegraphics[width=\linewidth]{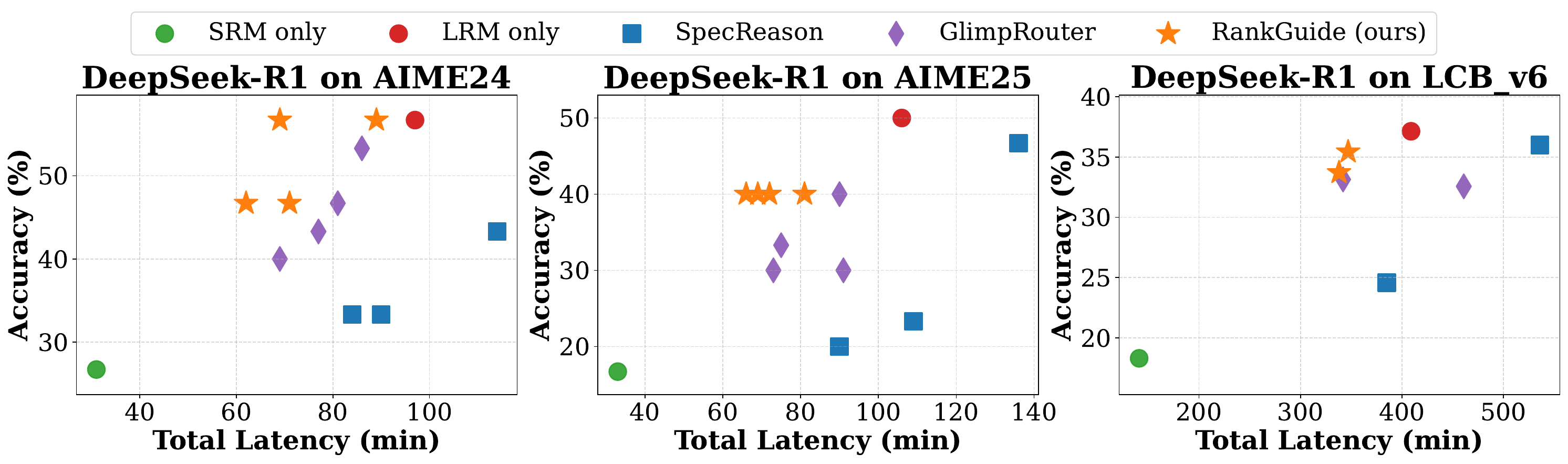}
    \caption{Accuracy vs. reasoning latency on AIME-24, AIME-25, LCB-v6 using DeepSeek-R1 (Qwen-1.5B / 32B) under multiple entropy thresholds for RankGuide and GlimpRouter and acceptance scores for Specreason.}
    \label{fig:abla-threshold}
\end{figure}

\subsection{Additional Experimental Setups}
\label{appendix:setup}
\textbf{Implementation Details.}
We implement all methods within a unified vLLM~\citep{kwon2023efficient} inference framework with consistent decoding configurations, and integrate our steering mechanism into vLLM. 
For DeepSeek-R1 models, we use greedy decoding, while for Qwen3 models we use temperature $0.6$ and top-$p$ $0.95$. 
The maximum generation length is set to 16{,}384 tokens for AIME and 8{,}192 tokens for LiveCodeBench \rev{and GPQA-Diamond}.
We evaluate accuracy using pass@1, i.e., whether the first generated solution is correct, and measure latency as the total wall-clock time required to produce the final answer, including all intermediate routing and verification steps. 
For both rank-guided routing and steering, we adopt a step-wise reasoning protocol, where routing decisions are made and steering vectors are applied at reasoning step boundaries identified by delimiter tokens (e.g., ``\texttt{\textbackslash n\textbackslash n}''). 
Our method computes routing signals based on a sliding window of recent hidden states with window size $W=10$, \rev{reshaped to $W \times 16 \times 16 \times d_{\text{hid}}/256$}, while entropy-based baselines rely on token-level output distributions.
\rev{Following GlimpRouter~\citep{zeng2026glimprouter}, the entropy signal is obtained from a single prefill pass with the preceding steps in context, taking the top-$k$ log-probabilities at the first token of the upcoming step with $k=20$ and renormalizing them to sum to one.}
We evaluate DeepSeek models on 4$\times$ A100 40GB GPUs, and Qwen3 models on 2$\times$ H100 GPUs.
\rev{All latency measurements for a given benchmark are collected on the same node without co-tenancy, and are reported as elapsed wall-clock time from the start of the first problem to the completion of the last, so that routing, steering and verification overheads are all included.}

\begin{figure}[!t]
    \centering
    \includegraphics[width=\linewidth]{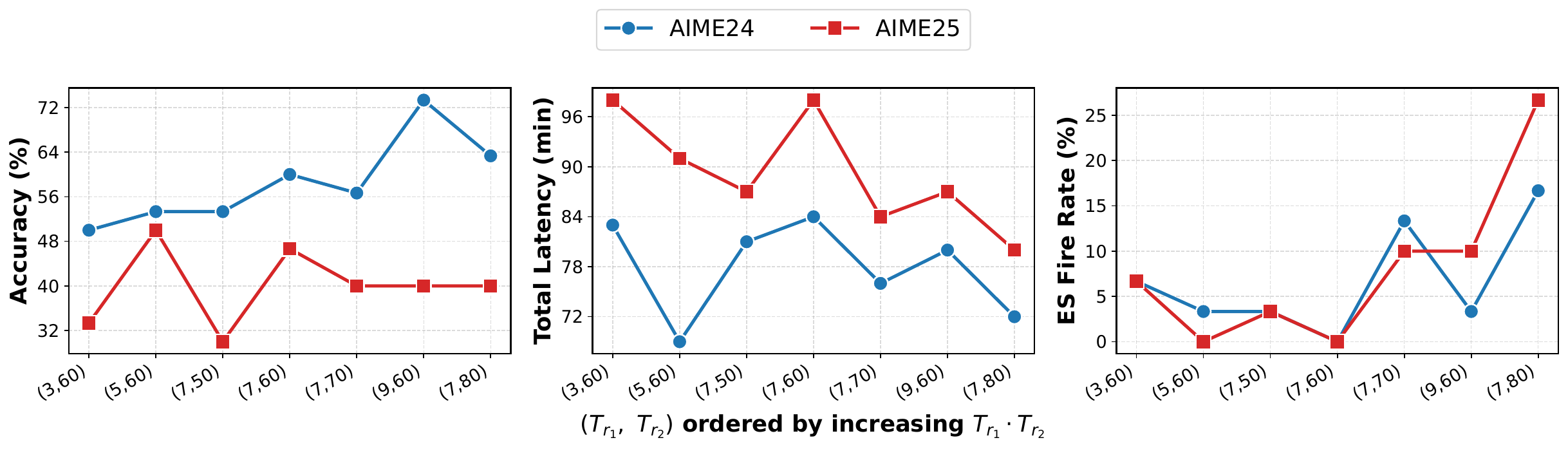}
    \caption{\rev{Ablation study on accuracy (left), latency (middle), and early-stopping (ES) fire rate (right) vs. rank threshold using AIME24 and AIME25 with the DeepSeek-R1 pair. Settings are ordered by increasing $T_{r_1} \cdot T_{r_2}$.}}
    \label{fig:rank-thresh}
\end{figure}

{\revon
\textbf{Routing threshold calibration protocol.}
The entropy threshold $T_e$ is selected from $\{0.7, 0.9, 1.0, 1.1\}$ following GlimpRouter~\citep{zeng2026glimprouter}, and we evaluate our method and GlimpRouter under the same settings for fair comparison.
On MATH-Train, the median minimum ranks are $(r_1, r_2) = (9.0, 79.0)$ over correct responses and $(5.0, 59.0)$ over incorrect responses, and we place $T_{r_1}$ and $T_{r_2}$ between the corresponding pair of medians (\S~\ref{sec:rank_guided_routing}).
The rank thresholds are therefore selected from $T_{r_1} \in \{6, 7, 8\}$ and $T_{r_2} \in \{50, 60, 70, 80\}$, ranges centred on this calibration-derived region, and the same rule was applied unchanged on all evaluation datasets.
For SpecReason~\citep{panspecreason}, the acceptance score $T_a$ is selected from $\{7, 8, 9\}$ following their paper.

\textbf{Steering vector calibration protocol.}
The steering vector is calibrated once on MATH-Train and applied to all evaluation datasets \textit{without} re-calibration.
This transfer is by design rather than by convenience: the vector is built from contrastive activations between execution and validation steps, and therefore isolates a reflection / over-validation direction that is a property of multi-step chain-of-thought behavior rather than of mathematical content.
Prior work empirically supports the same transfer, with SEAL~\citep{chen2025seal} and SkipKV~\citep{tian2025skipkv} both calibrating steering vectors on MATH-Train and applying them to LiveCodeBench without re-calibration at comparable performance.
We fix the steering factor $\alpha=1.0$ following SEAL~\citep{chen2025seal} as a calibration-free default.
\par}

\subsection{Additional Ablation Study on Routing Thresholds}
\label{appendix:routing-thresh}
As shown in Fig.~\ref{fig:abla-threshold}, we study the effect of different routing thresholds on the accuracy–latency trade-off for the DeepSeek-R1 model pair across AIME-24, AIME-25, and LCBv6.
Specifically, the routing thresholds are selected from $T_e \in \{0.7, 0.9, 1.0, 1.1\}$.
For fair comparison, we evaluate both our method and GlimpRouter under the same entropy threshold settings.
For SpecReason, the acceptance score is selected from $T_{a} \in \{7, 8, 9\}$.
Across all three benchmarks, RankGuide consistently forms the Pareto frontier, achieving higher accuracy at comparable or lower latency than all baselines.
In contrast, SpecReason incurs substantially higher latency while yielding only modest accuracy improvements, and can even be slower than the LRM baseline.
While GlimpRouter performs better than SpecReason, its purely entropy-based routing leads to sub-optimal behavior; RankGuide improves upon it by routing both uncertain and over-confident steps, reducing redundant re-validation, and achieving a better accuracy–latency trade-off.

\begin{figure}[!t]
    \centering
    \includegraphics[width=0.62\linewidth]{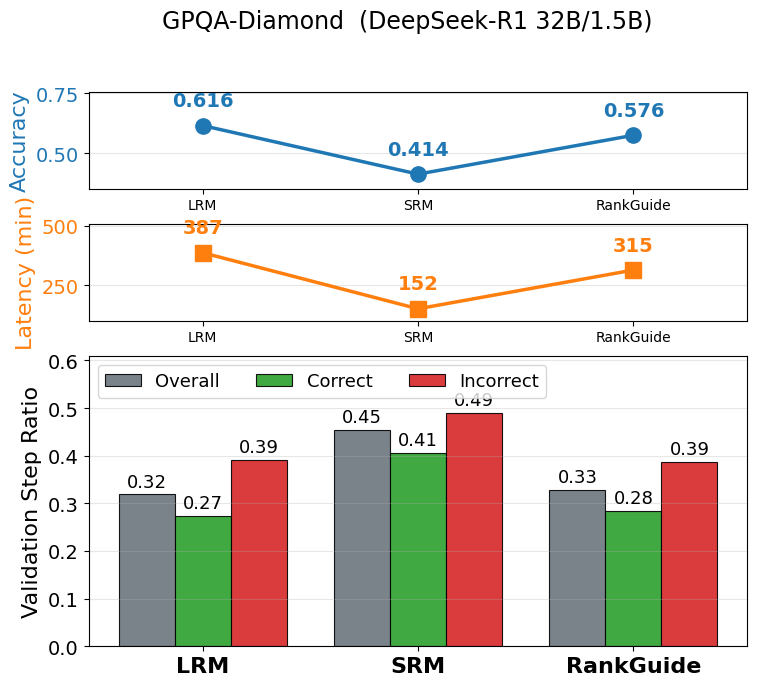}
    \caption{\rev{Comparison of accuracy, latency and validation step ratio on GPQA-Diamond with the DeepSeek-R1 pair.}}
    \label{fig:gpqa-revalidation}
\end{figure}

\subsection{Additional Ablation Study on Rank Thresholds}
\label{appendix:rank-thresh}
{\revon
We next study the effect of the rank thresholds on accuracy, latency and early-stopping fire rate.
Fig.~\ref{fig:rank-thresh} sweeps seven $(T_{r_1}, T_{r_2})$ combinations spanning $T_{r_1} \in \{3, 5, 7, 9\}$ and $T_{r_2} \in \{50, 60, 70, 80\}$ at $T_e = 0.9$ on AIME24 and AIME25 with the DeepSeek-R1 pair, deliberately extending beyond the calibration-derived region in both directions.

Three observations follow.
First, accuracy never collapses when the rank thresholds are selected by the calibration protocol of \S~\ref{sec:rank_guided_routing}: the two swept settings that satisfy it, $(7, 60)$ and $(7, 70)$, reach $60.0\%$ and $56.7\%$ on AIME24 and $46.7\%$ and $40.0\%$ on AIME25.
The settings that do collapse are precisely those the protocol excludes: $(7, 50)$ and $(3, 60)$, whose thresholds sit at or below the incorrect-response medians, fall to $30.0\%$ and $33.3\%$ on AIME25.
The protocol therefore screens out the failure region rather than merely landing inside a flat one, supporting its use in place of per-dataset threshold tuning.
Second, the early-stopping fire rate grows with $T_{r_1} \cdot T_{r_2}$, reaching $16.7\%$ on AIME24 and $26.7\%$ on AIME25 at the loosest setting, and the lowest-latency point on each dataset is the setting where it fires most.
Early stopping is thus the mechanism by which looser thresholds buy latency, though the relationship is not monotone: the tightest settings admit few collapsed trajectories, so their latency is governed by per-step routing instead.
Third, early stopping never dominates the policy, terminating at most $26.7\%$ of problems, so per-step routing governs the decision on the large majority and the two stages remain complementary.
Some settings outside the calibration region attain higher accuracy than the default reported in Table~\ref{tab:main_results}, such as $(9, 60)$ on AIME24 and $(5, 60)$ on AIME25.
We nonetheless report the calibration-derived setting, which avoids extreme hyperparameter search on the evaluation set and keeps the comparison fair.
\par}

\subsection{Routing Overhead Analysis}
\label{appendix:overhead}
\rev{Computing a TT decomposition at every reasoning step is a natural efficiency concern, so we quantify it by instrumenting the rank computation at every routing decision across all AIME24 runs with the DeepSeek-R1 pair. The mean per-step TT-SVD takes $5.9$\,ms, ${\approx}1.2\%$ of end-to-end inference latency. The cost is stable across $(T_{r_1}, T_{r_2}, T_e)$, i.e.\ it does not grow with task difficulty or hyperparameter choice, because the tensor shape is fixed by $W$ and $d_{\text{hid}}$ rather than by the input. All latencies reported in the paper include this overhead.}

\subsection{Revalidation Profile under RankGuide}
\label{appendix:revalidation}
{\revon
Table~\ref{tab:main_results} shows that RankGuide lowers the aggregate validation step ratio, and we further reproduce the analysis of Fig.~\ref{fig:SRM motivation} (right) under RankGuide on GPQA-Diamond, shown in Fig.~\ref{fig:gpqa-revalidation}.
Two findings follow.
First, the incorrect $>$ correct revalidation gap holds under all three systems.
Second, RankGuide compresses the SRM's indiscriminate revalidation profile ($0.45$ / $0.41$ / $0.49$ for overall / correct / incorrect) to essentially the LRM's selective profile ($0.33$ / $0.28$ / $0.39$ against the LRM's $0.32$ / $0.27$ / $0.39$), while running $18.6\%$ faster than the LRM.
Together, these results indicate that RankGuide removes most of the SRM's redundant validation while preserving the LRM's correct-versus-incorrect verification structure, yielding a better accuracy--latency trade-off.
\par}

\subsection{Sensitivity to Tensor Factorization}
\label{appendix:shape}
\begin{table}[t]
\centering
\revon  
\caption{Sensitivity of the hidden-feature rank $r_2$ to the reshaping factorization, measured as
the median $r_2$ over correct and incorrect MATH-Train responses. 
}
\setlength{\tabcolsep}{6pt}
\begin{tabular}{lcccc}
\toprule
 & \multicolumn{2}{c}{\textbf{DeepSeek-R1-Distill-Qwen-1.5B}} & \multicolumn{2}{c}{\textbf{Qwen3-4B}} \\
\cmidrule(lr){2-3}\cmidrule(lr){4-5}
\textbf{Tensor shape} & Correct & Incorrect & Correct & Incorrect \\
\midrule
$16 \times 16 \times d_{\text{hid}}/256$ & 79.0 & 59.0 & 83.0 & 71.5 \\
$8 \times 8 \times d_{\text{hid}}/64$ & 62.0 & 37.0 & 53.0 & 42.0 \\
\bottomrule
\end{tabular}
\label{tab:shape}
\end{table}

{\revon
Reshaping the hidden-state window into a higher-order tensor is a design choice, and the rank values our thresholds are compared against depend on it.
We therefore report how that choice is made and what changes under an alternative.
We set $d_1 = d_2 = 16$ for two reasons: it keeps the three feature modes close in size, so no single mode dominates the TT-cores, and $256$ divides $d_{\text{hid}}$ for both SRMs we evaluate ($1536$ and $2560$), so one factorization serves both model families without per-model tuning.
Table~\ref{tab:shape} repeats the rank-distribution study of \S~\ref{sec:motivation1} under an alternative $8 \times 8$ factorization.
The reasoning-step rank $r_1$ is unaffected by construction: the first TT-SVD step of Algorithm~\ref{alg:tt_rank} determines it from the window size $W$, not from the feature factorization.
The hidden-feature rank $r_2$ does shift, as expected, but the correct $>$ incorrect ordering that our routing signal relies on is preserved under both factorizations and on both model families.
The calibration rule therefore transfers directly: moving from $16 \times 16$ to $8 \times 8$ shifts the recommended $T_{r_2}$ range from approximately $[60, 80]$ to $[40, 60]$, while the rule that generates it -- place the threshold between the correct- and incorrect-response medians -- is unchanged.
\par}


\end{document}